%% file: main.tex
\documentclass{article}

\usepackage{microtype}
\usepackage{graphicx}
\usepackage{booktabs} 

\usepackage[utf8]{inputenc} 
\usepackage[T1]{fontenc}    
\usepackage{hyperref}       
\usepackage{url}            
\usepackage{amsfonts}       
\usepackage{nicefrac}       
\usepackage{microtype}      
\usepackage{xcolor}         
\usepackage{multirow}
\usepackage{multicol}
\usepackage{graphicx}
\usepackage{comment} 

\usepackage{hyperref}

\usepackage[preprint]{icml2023}

\usepackage{amsmath}
\usepackage{amssymb}
\usepackage{mathtools}
\usepackage{amsthm}

\usepackage[capitalize,noabbrev]{cleveref}

\theoremstyle{plain}
\newtheorem{theorem}{Theorem}[section]

\newtheorem{lemma}[theorem]{Lemma}
\newtheorem{corollary}[theorem]{Corollary}
\theoremstyle{definition}

\theoremstyle{remark}

\usepackage[textsize=tiny]{todonotes}

\icmltitlerunning{Neural Attention Memory}

\begin{document}

\twocolumn[
\icmltitle{Neural Attention Memory}



\icmlsetsymbol{equal}{*}

\begin{icmlauthorlist}
\icmlauthor{Hyoungwook Nam}{uiuc}
\icmlauthor{Seung Byum Seo}{uiuc}
\end{icmlauthorlist}

\icmlaffiliation{uiuc}{Department of Computer Science, University of Illinois at Urbana-Champaign, Urbana, Illinois}

\icmlcorrespondingauthor{Hyoungwook Nam}{hn5@illinois.edu}
\icmlcorrespondingauthor{Seung Byum Seo}{sbseo2@illinois.edu}

\icmlkeywords{Deep Learning, Machine Learning, Neuro-symbolic AI, Memory augmented neural network, Few-shot learning, Transformer, ICML}

\vskip 0.3in
]



\printAffiliationsAndNotice{}  

\begin{abstract}

We propose a novel perspective of the attention mechanism by reinventing it as a memory architecture for neural networks, namely Neural Attention Memory (NAM).
NAM is a memory structure that is both readable and writable via differentiable linear algebra operations.
We explore three use cases of NAM: memory-augmented neural network (MANN), few-shot learning, and efficient long-range attention.
First, we design two NAM-based MANNs of Long Short-term Memory (LSAM) and NAM Turing Machine (NAM-TM) that show better computational powers in algorithmic zero-shot generalization tasks compared to other baselines such as differentiable neural computer (DNC).
Next, we apply NAM to the N-way K-shot learning task and show that it is more effective at reducing false positives compared to the baseline cosine classifier.
Finally, we implement an efficient Transformer with NAM and evaluate it with long-range arena tasks to show that NAM can be an efficient and effective alternative for scaled dot-product attention.
\end{abstract}

\section{Introduction}
\label{sec:intro}
\input{1_intro}

\section{Background}
\label{sec:background}
\input{2_background}

\section{Method}
\label{sec:method}
\input{3_method}

\section{Experimental Setup}
\label{sec:setup}
\input{4_setup}

\section{Experimental Result}
\label{sec:result}
\input{5_result}

\section{Possible Future Work}
\label{sec:discussion}
\input{6_discussion}

\section{Conclusion}
\input{7_conclusion}

\bibliography{references}
\bibliographystyle{icml2023}



\end{document}

%% file: 1_intro.tex
Scaled dot-product attention~\citep{transformer} has become a core mechanism of state-of-the-art deep learning models for variety of machine learning tasks, including natural language processing~\citep{bert}, multi-modal task~\citep{visualbert}, and graph data processing~\citep{graphsage}.
Specifically, the Transformers using the self-attention method have replaced recurrent neural networks (RNN) by outperforming them in most of the tasks.

The attention mechanism can be understood as a read-only memory.
Given a sequence of values, the attention reads a specific value by computing a score for each value and conducting the weighted sum of the values with their scores.
In this paper, we reconstruct the attention mechanism as a memory that is both readable and writable, namely \emph{neural attention memory} (NAM).
Using the same query-key-value setup of the scaled dot-product attention, we define read and write primitives to a \emph{memory matrix}.
The write ($WR$) operation overwrites the value vector mapped to the unit key vector in the memory matrix by adding the outer-product of two.
Then the read ($RD$) operation reads the value vector mapped to the unit query vector simply by matrix-vector multiplication.
We mathematically prove that after the write operation, reading the memory matrix with the same key outputs the most recently written value.
All of the operations are based on differentiable and efficient linear algebra operations so that NAM-augmented neural networks can be end-to-end trained using backpropagation.

Among many possible applications of this differentiable memory structure, we explore three use cases.
First, we design a memory-augmented neural network (MANN) with NAM to solve algorithmic tasks.
We propose two MANN structures using NAM: Long Short-term Attention Memory replaces LSTM~\cite{lstm}'s long-term memory with a memory matrix, and NAM Turing Machine (NAM-TM) implements a differentiable Turing tape using NAM primitives.
We compare LSAM and NAM-TM to others in algorithmic tasks of number sequence prediction~\citep{numbersequenceprediction} and sequence reduction.
Specifically, we test their inductive zero-shot generalization capability in length by training the models with sequences of limited length and validating them with longer sequences unobservable during training.
Transformers often fail to generalize such algorithmic rules as they lack inductive bias~\cite{universaltransformer} and are stateless.
The evaluation results show that their computational powers are superior to other baselines, including Universal Transformer~\citep{universaltransformer} and DNC~\citep{dnc}.
While the generic LSAM model consistently outperforms the others, NAM-TM shows even better results at algorithmic tasks.
The results indicate that NAM is a powerful method to implement memory in neural networks.

Second, we propose using NAM for few-shot learning.
Writing to a memory matrix maps the input-output pair in one-shot in that memorization with NAM can be a good few-shot learning method.
We propose a $N$-way $K$-shot learning mechanism via memorization to a memory matrix.
We apply NAM based few-shot learning method to the $N$-way $K$-shot learning setup with MiniImageNet dataset~\cite{fewshotwithoutforgetting}.
NAM's performance is mostly on par with the state-of-the-art feature-averaging cosine classifier as it can be understood as a special case of NAM based few-shot learning.
NAM shows better classification results when both novel and base classes are present, suggesting that NAM's erasure functionality is effective at reducing false positives.

Lastly, we construct an efficient Transformer structure based on NAM.
Scaled dot-product attention has computational complexity becomes quadratic to the length of the sequence.
We present an efficient attention mechanism based on NAM that has linear complexity to the sequence length, namely normalized outer-product attention.
We evaluate NAM-based efficient Transformer in long-range arena~\citep{longrangearena} tasks.
Its efficacy is on par with Transformer and Linear Transformer, implying that NAM can be an efficient alternative to the scaled dot-product attention.

The main contributions of this work are as follows:
\begin{itemize}
    \setlength{\itemsep}{0pt}
    \item We re-invent the attention mechanism as a memory architecture for neural networks that is both readable and writable, namely neural attention memory (NAM).
    \item Using NAM's read and write primitives, we propose NAM-augmented neural network designs, a few-shot learning method, and an efficient linear attention.
    \item We theoretically and experimentally prove that NAM can be used for differentiable memory structure, few-shot learning, and efficient attention.
\end{itemize}


%% file: 2_background.tex
\subsection{Scaled dot-product attention}
Attention mechanisms of deep neural networks~\citep{bahdanau,luong} provide differentiable methods of selectively attending items from a variable-length sequence.
While there are multiple variations of attention mechanism, most of them share the same high-level structure: 1) compute the attention scores of the items, and 2) return the weighted sum of their vector representations using the scores.
Among the variations, \textit{scaled dot-product attention}~\citep{transformer} has been the most successful.
For each token, there are a key vector and a value vector associated to it.
Given a query vector, the scores are determined by the scaled dot-product of the query and the keys.
Then the output is computed by weighted sum of softmaxed scores and the value vectors.

The self-attention mechanism based on the scaled dot-product attention has proven to be very powerful, replacing the needs of RNNs.
Since attention reaches every element of a sequence in an $O(1)$ path, it avoids the vanishing gradient problem~\citep{vanishing}, enjoys high parallelism, and allows huge models with deeply stacked layers~\citep{bert, gpt3}.
However, its stateless and parallel architecture also bring multiple limitations.
First, the computational cost quadratically increases with the sequence length, making it inefficient for long-range contexts~\citep{longformer,performer,reformer,lineartransformer} and edge inference environments~\citep{edgebert}.
Also, the memory-less architecture lacks inductive bias~\citep{universaltransformer}, making it impossible to generalize inductive algorithmic rules~\citep{seq2grid}.
There are researches to resolve those issues, but a practical solution is yet to be found.

\subsection{Memory-augmented neural network}
In theory, RNNs are proven to be as powerful as Turing machines~\citep{rnnturing}.
However, they fail to learn algorithmic patterns that require at least pushdown automata in practice~\citep{numbersequenceprediction}.
Therefore, there have been efforts to augment external memory architecture to neural networks, as known as memory-augmented neural networks (MANN)~\citep{dequernn,stackrnn,ntm,dnc}.
The main challenge for MANNs is to design differentiable read and write functions that can be trained via back-propagation.
Some MANNs use attention mechanism for differentiable read/write functions.
For instance, Neural Turing Machine (NTM)~\citep{ntm} and Differentiable Neural Computer (DNC)~\citep{dnc} leverage attention mechanism for implementing content-based addressing.
However, the addressing mechanisms are often not powerful enough so that they need to be augmented with complex extras, such as a link matrix with $O(N^2)$ cost.
Hence, MANNs often become inefficient and complex so that they are considered impractical outside of the algorithmic task domain.

\subsection{Transformer and Memory}
Transformers' attention mechanism can be understood as a memory where all hidden activations are stored and selectively read.
Since storing the entire sequence is inefficient, there have been studies to efficiently scale to longer context by adding memory architectures to Transformers.
Transformer XL~\citep{transformerxl} has attempted to overcome the limited context length of Transformers by using segment-level recurrence and relative positional encoding.
Compressive Transformer~\citep{compressivetransformer} and $\infty$-former~\citep{infiniteformer} further have extended the idea by compressing the recurrent segment into a smaller memory and using less-precise unbounded continuous memory, respectively.
While these works effectively address the scalability problem of Transformers, they have not extended their ideas to generic memory architecture that can be freely read and written.

\subsection{Few-shot Learning}

Few-shot learning aims to train an ML model with only a few examples.
Transformers with huge numbers of parameters are known to be good few-shot learners~\cite{gpt3}, but the reason for this is yet unknown.
Among many different few-shot learning definitions, we focus on \emph{N-way K-shot} classification setup~\cite{matching}.
In this setup, a network is given base classes with plenty of samples and $N$ novel classes with $K$ samples each.
The state-of-the-art approach for this is to use a feature-averaging cosine classifier~\cite{fewshotwithoutforgetting,fewshotandcontinual}.
Given a feature vector extracted from the feature extractor model, the classification score is determined by its cosine similarity between the per-class weight vectors of the classifier.
The weight vectors for the $N$ classes are based on the average of the feature vectors of their $K$ samples.

%% file: 3_method.tex
\subsection{NAM Read and Write Primitives}
\label{sec:namprimitives}
The main idea of neural attention memory is implementing an attention mechanism via matrix-vector multiplication of a \emph{memory matrix} $M\in \mathbb{R}^{d_v\times d_k}$, a \emph{unit query vector} $q \in \mathbb{R}^{d_k}$ and a \emph{read probability} $0 \leq p_r \leq 1$. 
Hereby $d_v$ and $d_k$ are feature dimensions of value and key vectors respectively.
Then, the read operation ($RD$) of NAM computes the read vector $r\in \mathbb{R}^{d_v}$ as follows.
$$ r = RD(M,q,p_r) = p_rMq $$

The memory matrix $M$ is written by adding \emph{outer products} of unit key vectors $k_i$ and value vectors $v_i$.
The write operation $WR$ updates the memory matrix with a write probability $p_w$ and an erase probability $p_e$ as below.
$$ M' = WR(M,k,v,p_w,p_e) = M + p_w vk^\top - p_e Mkk^\top $$
This write operation guarantees that reading with the same key vector yields the most recently written value.
Such a property can be mathematically proven as below.
\begin{theorem}
\label{thm:readwrite}
If \(k\) is a unit vector and \(M' = WR(M,k,v,1,1)\), \(RD(M',k,1)\) = \(v\). 
\end{theorem}
\begin{proof} 
$M'k = Mk + vk^\top k - Mkk^\top k = v$
\end{proof}
The verbal explanations of this theorem are as follows. 
First, $Mk = RD(M,k,1)$ yields the value in associated to $k$ so that $-p_e Mkk^\top$ \emph{erases} out the associated value from $M$.
Then, adding $p_w vk^\top$ overwrites the new value on top of it.

In a sequence of length $S$, if the keys $k^{(1)}, \dots, k^{(S)}$ are orthonormal, reading to the memory matrix can act as an attention mechanism.
\begin{lemma}
\label{lemma:orthonormal}
Let $k^{(1)}, \dots, k^{(S)}$ be orthonormal and $M^{(0)} = \boldsymbol{0}^{d_v \times d_k} $. If $M^{(i+1)} = WR(M^{(i)}, k_{i+1}, v_{i+1}, 1, 1)$, then $M^{(S)} = \sum_i{v_i k_i^\top}$.
\end{lemma}

\begin{proof} 
Trivially, $M^{(1)} = v_1 k_1^\top $. \\
If $M^{(i)} = v_1 k_1^\top + \dots + v_i k_i^\top$, then \( M^{(i)}k_{i+1} =  v_1 k_1^\top k_{i+1} + \dots + v_i k_i^\top k_{i+1} = 0\) as the keys are orthonormal. Hence,
$WR(M^{(i)}, k_{i+1}, v_{i+1}, 1, 1) = M^{(i)} + v_{i+1} k_{i+1}^\top + \boldsymbol{0} = v_1 k_1^\top + \dots + v_i k_i^\top + v_{i+1} k_{i+1}^\top$. \\
By mathematical induction, the lemma holds.
\end{proof}
The following theorem derived by the lemma shows that reading the memory matrix can be an attention mechanism.
\begin{theorem}
\label{thm:read}
$\forall i \in 1,\dots,S,  RD(M^{(S)},k^{(i)},1) =  v^{(i)}$
\end{theorem}
\begin{proof} 
$M^{(S)}k^{(i)} = \sum_t v^{(t)}k^{(t)\top}k^{(i)} = v^{(i)}$ as $ k^{(t)\top} k^{(i)}$ is 1 when $t=i$ and 0 otherwise.
\end{proof}

At each time step, the computational and space complexity of NAM are both $O(d_v d_k)$.
Since $d_v$ and $d_k$ are per-head dimensions, it can be further reduced by the factor of the number of the heads $H$, to  $O(H (d_v/H) (d_k/H)) = O(d_v d_k / H)$.
On the other hand, scaled dot-product attention's compute and space complexities are $O(S (d_v + d_k))$, if the sequence has $S$ tokens.
This opens up the opportunity for using NAM as an efficient linear attention mechanism for long-range sequential tasks or edge inference environments.
Furthermore, the outer-product operation $vk^\top$ for NAM is often more efficient than dot-products in modern parallel HW since its compute complexity $O(d^2)$ is much higher than the amount of memory access required, $O(d)$.

\subsection{Long Short-term Attention Memory}
The first MANN design we propose is a generic recurrent neural network that is derived from Long Short-term Memory (LSTM)~\citep{lstm}.
LSTM leverages two recurrent state vectors: the short-term hidden state and the long-term cell state.
To mitigate the problems of vanishing/exploding gradients, the cell state is additively updated using forget and input gates.
Then, the output gate selectively reads the cell state to determine the hidden state.
Long Short-term Attention Memory (LSAM) follows the same principle.
Instead of using the vector cell state, it leverages the memory matrix $M_t\in \mathbb{R}^{d^2}$ which is also additively updated using the NAM write primitive.
The hidden state $h_t\in \mathbb{R}^{d}$ is retrieved by reading the memory matrix.
Given the input $x_t\in \mathbb{R}^{d}$, the update rule $M_t, h_t = LSAM(x_t,M_{t-1},h_{t-1})$ is defined as follows.
\begin{align*}
    [q_t:k_t:v_t] &= W_{qkv}[x_t:h_{t-1}] + b_{qkv}\\
    <p_r,p_w> &= \sigma (W_{rw}[x_t:h_{t-1}] + b_{rw}) \\
    M_t &= WR(M_{t-1},\mu(k_t),v_t,p_w,p_w)\\
    h_t &= RD(M_t,\mu(q_t^i),p_r)
\end{align*}
Hereby $:$ is the concatenate operator, $\sigma(.)$ is the sigmoid function, and $W_{qkv}, W_{rw}, b_{qkv}, b_{rw}$ are trainable weights and biases.
Although the memory matrix $M_t$ has much higher capacity than a vector cell state, the computational complexity of LSAM is identical to that of LSTM.
This is because $WR$ and $RD$ have complexity of $O(d^2)$ which is identical to that of matrix-vector multiplication of the weights and the states.
Like Transformers, we can design multi-headed LSAM by concatenating per-head states $M_t^i$ and $h_t^i$.
Bidirectional LSAM is also possible by splitting the multiple heads into two directions.
The backward heads are updated in the opposite direction by the rule of $ M_t^j, h_t^j = LSAM(x_t, M_{t+1}^j, h_{t+1}^j)$.

The LSAM architecture combines the strengths of recurrent neural networks (RNN) and Transformers.
Since LSAM follows the RNN design so that it enjoys strengths of RNNs such as low computational cost for inference and recurrent inductive bias.
Additionally, it benefits from the strengths of attention because reading and writing the memory matrix natively incorporates the attention mechanism.

\subsection{NAM Turing Machine}

The second MANN design to present is a network that resembles Turing machine.
Neural Turing Machine (NTM) ~\citep{ntm} is one of the early neural networks that implement external memory structure with differentiable read and write methods.
It is a basis of Differentiable Neural Computer (DNC)~\citep{dnc} which has proven to be effective at solving a variety of algorithmic tasks such as answering synthetic questions and finding shortest paths.
Their external memory matrix is accessed by read and write heads using differentiable attention mechanisms.

Hereby we design NAM Turing Machine (NAM-TM) which adopts the design principles of NTM.
The main idea of NAM-TM is to treat the tape state with $N$ vectors $T = [v_1, v_2, ... v_N, \boldsymbol{0}, ...] \in \mathbb{R}^{L\times d}$ as a memory matrix as follows.
\begin{align*}
    T &= [v_1, ... v_n, \boldsymbol{0}, ...] = \sum_i{v_i e_i^\top} & (v_i\in \mathbb{R}^d, e_i\in \mathbb{R}^L)
\end{align*}
Hereby $e_i$ are standard basis vectors, e.g., $<0,0,...1,...,0>$ of $\mathbb{R}^L$ where $L$ is the size of the tape.
This tape state can now be accessed by using NAM read/write primitives.

NAM-TM is a differentiable function that takes the tape state $T$, read and write heads $H_r,H_w \in \mathbb{R}^L$, and the input vector $x\in \mathbb{R}^d$ and produces the read output $R \in \mathbb{R}^d$ along with the updated tape and head states $T',H_r',H_w'$.
$$R, T', H_r', H_w' = NAMTM(T,H_r,H_w,x)$$
The read and write heads  are positional vectors to attend the memory matrix for reading and writing the states.
At each time step, the positions can be updated by four actions: LEFT, RIGHT, NO-OP, and JUMP.
They are controlled by a controller neural network $nn\_control(x)$ which emits read and write probabilities $p_r,p_w$, action probabilities $p_{right},p_{left},p_{noop},p_{jump}$ for each head and a unit jump query vector $q_{jump}\in \mathbb{R}^d$.
Given the controller outputs and the value $v = W_v x$ to write, reading and writing the memory are conducted by NAM primitives as follows.
\begin{align*}
    R &= RD(T,H_r,p_{r}) \\
    T' &= WR(T,H_w,v,p_{w},p_{w})
\end{align*}
Then, each head is updated to the next position based on the action probabilities.
LEFT and RIGHT actions can be performed by the differentiable $roll$ function, which is a linear transformation $\mathbb{R}^L \longrightarrow \mathbb{R}^L$ mapping $e_i$ to $e_{i+1}$.
$$(H_{LEFT},H_{RIGHT}) = (roll^{-1}(H),roll(H))$$

The jump position $H_{JUMP}$ is determined by reading the transpose of a key tape $K$ with the unit query vector $q_{jump}$.
The key tape is written in the same way as the tape $T$, but it stores the corresponding unit key vector derived using the weight matrix $W_k \in \mathbb{R}^{d \times d}$ and the L2 normalization $\mu(.)$.
\begin{align*}
    K' &= WR(K, H_w, \mu(W_k x), p_{w}, p_{w}) \\
    H_{JUMP} &= RD(K'^\top,q_{jump},1)
\end{align*}
One can understand the transpose of the key tape $K^\top = \sum_i e_i k_i^\top$ as a jump table.
Reading $K^\top$ with a key vector $k_i$ returns the corresponding position vector $e_i$ if the keys are orthonormal.
Finally, the next head position $H'$ is updated as a weighted sum of the positions.
\begin{align*} 
    H' =& p_{noop}\times H + p_{left}\times H_{LEFT}\\
    &+ p_{right}\times H_{RIGHT} + p_{jump}\times H_{JUMP}
\end{align*}

While all of the computations are technically done with a fixed tape length $L$, none of the trainable parameters depend on the value of $L$.
That is, a NAM-TM trained on certain length $L$ can be applied to any tape length $L'$ without modification nor re-training.
Theoretically, it can be extended to infinite-dimensional Hilbert spaces.

There are multiple strengths in NAM-TM design compared to the memory structures of NTM and DNC.
First, unlike NTM and DNC, NAM-TM's building blocks are simple and computationally efficient.
Second, NAM-TM's addressing mechanism is based on the query-key-value attention mechanism of NAM, which is more powerful than the content-based attention mechanism used in NTM and DNC.
Finally, NAM-TM's design is flexible in that it is easy to add/remove transition rules of the read/write heads.
For example, the JUMP transition rule is optional in that a Turing machine only requires LEFT and RIGHT transitions in theory.
One can also add another transition rule for the head positions if the rule can be defined with differentiable functions.

\subsection{Few-shot learning}
Intuitively, a human can perform few-shot learning via memorizing.
Therefore, we propose one-shot or few-shot learning as memorization using NAM primitives.
Given a input-output vector pair $(x,y)$, it is possible to conduct \emph{one-shot learning} of a weight matrix $W$ by the $WR$ operation.
This can be shown by the following corollary of Theorem~\ref{thm:readwrite}.
\begin{corollary}
If \(W'=WR(W,y/|x|,x/|x|,1,1)\), \\ then \(W'x = y\).
\end{corollary}

In fact, a feature-averaging cosine classifier~\cite{fewshotwithoutforgetting} is a special case of NAM, where the erasure probability is set to zero.
Assume $W_{novel}^0\in \mathbb{R}^{N\times H}$ is the cosine classifier for $N$ novel classes and $H$-dimensional feature vectors initialized to zero.
The following theorem shows that the feature-averaging cosine classifier can be defined using the NAM primitives.
\begin{theorem}
Assume \(e_i \in \mathbb{R}^N\) is a one-hot vector and $h_i$ is the average of feature vectors for class \(i\). If \( W_{novel}^{i} = WR(W_{novel}^{i-1},e_i, h_i/|h_i|,1,0)\), then \(W_{novel}^{N} h/|h|\) gives a vector of cosine similarities between \(h_i\) and \(h\).
\end{theorem}

\begin{proof}
\(W_{novel}^N = \sum_i e_i h_i^\top / |h_i|\). Hence,
\begin{align*}
W_{novel}^{N} (h/|h|) &= \sum_i (h_i^\top h/|h_i||h|)e_i \\ 
&= <h_1^\top h/|h_1||h|, ..., h_N^\top h/|h_N||h|>. 
\end{align*}
\end{proof}

This feature-averaging cosine classifier algorithm does not alter the base-class classifier.
While this prevents catastrophic forgetting, it can suffer from false-positives to the base categories.
Using NAM, we can leverage the erasure probability to supress the false-positives to the base classes.

\begin{algorithm}[ht]
   \caption{NAM Few-shot Learning}
   \label{alg:namfewshot}
\begin{algorithmic}
  \STATE {\bfseries Input:} $W_{base}$, feature vectors $x_{11} \dots x_{NK} \in \mathbb{R}^H$
  \STATE $W = concat(W_{base}, \boldsymbol{0}^{N\times H}) \in \mathbb{R}^{(B+N)\times H}$
  \STATE $W' = \boldsymbol{0}^{(B+N)\times H}$
  \FOR{$n=1$ {\bfseries to} $N$}
  \FOR{$k=1$ {\bfseries to} $K$}
  \STATE $p_w, p_e = \sigma (f_{rw}(x_{nk}))$
  \STATE $W_{nk} = WR(W, e_n, x_{nk}/|x_{nk}|, p_w, p_e)$
  \STATE $W' = W_{nk}/K$
  \ENDFOR
  \ENDFOR
  \STATE {\bfseries return} $\mu (W')$

\end{algorithmic}
\end{algorithm}

Algorithm~\ref{alg:namfewshot} describes a generalized $N$-way $K$-shot learning given pre-trained feature extractor $feat(\cdot)$ and the $B$ base-class classifier $W_{base} = e_1 w_1^\top + \dots + e_B w_B^\top$.
Hereby $\sigma(\cdot)$ is the sigmoid activation, $f_{rw}(\cdot)$ is a trainable neural network for read and write probabilities, and $\mu(\cdot)$ is L2 normalization per row.
Unlike the feature-averaging consine classifier, the write probability $p_w$ gives a relative importance to each sample and the erase probability $p_e$ suppresses the false-positive from the base-class classifier.
Note that the loop can be parallelized so that the actual implementation uses matrices to compute the entire $N\times K$ batch at once.

\subsection{Efficient Transformer using NAM}
\label{sec:namtransformer}

As discussed in Section~\ref{sec:namprimitives}, NAM is a good candidate for an efficient attention mechanism as its computational complexity is linear to the sequence length.
However, the sequential nature of $WR$ can make NAM inefficient in parallel GPU and NPU architectures.
This is because the erasure part $- p_e Mkk^\top$ of $WR$ depends on the current state of $M$.
Hence, we can make $WR$ parallel by setting the erase probability $p_e$ to zeros.
The reads can also become parallel if we make it \emph{bidirectional} by reading the same final memory matrix $M = \sum_t p_w^{(t)} v^{(t)}k^{(t)\top}$, which is sum of outer product from every time step $t \in 1,\dots,S$.
Reading this memory matrix can act as an attention mechanism, namely \emph{normalized outer-product attention}.

As proven at Theorem~\ref{thm:read}, normalized outer-product attention can fully replace other attention mechanisms if the keys are orthonormal.
In practice, such a hard guarantee is not possible.
However, normalized outer-product attention can act as a soft attention mechanism if the keys are well separated from each other.
By removing the erasure process, it trades off the computational power of NAM for the maximal efficiency and parallelism.

Hereby we implement an efficient linear bidirectional transformer using normalized outer-product attention, namely \emph{NAM-Transformer}.
NAM-Transformer leverages the fully parallel implementation of  normalized outer-product attention and sets the read/write probabilities to 1.
Given queries $Q\in \mathbb{R}^{S \times d_k}$, keys $K\in \mathbb{R}^{S \times d_k}$, and values $V\in \mathbb{R}^{S \times d_v}$ of a sequence with $S$ tokens, NAM-Transformer's self-attention is computed as follows.
$$ SelfAttn_{NAM}(Q,K,V) = (V^\top \mu(K)) \mu(Q)$$
Hereby $\mu$ is a unit-vector normalization function applied to each row vector of $K$.
Computationally, this is similar to Linear Transformer~\citep{lineartransformer}.
The differences from Linear Transformer to ours are: 1) we use unit vector normalization instead of ELU kernel function, 2) so that we do not need to compute the causal masking factor $Z$~\citep{lineartransformer}.
That is, one can think Linear Transformer as a special variant of NAM, sacrificing computational capabilities for parallelism and efficiency.


%% file: 4_setup.tex
\begin{figure}[ht]
    \centering
    \includegraphics[width=0.8\linewidth]{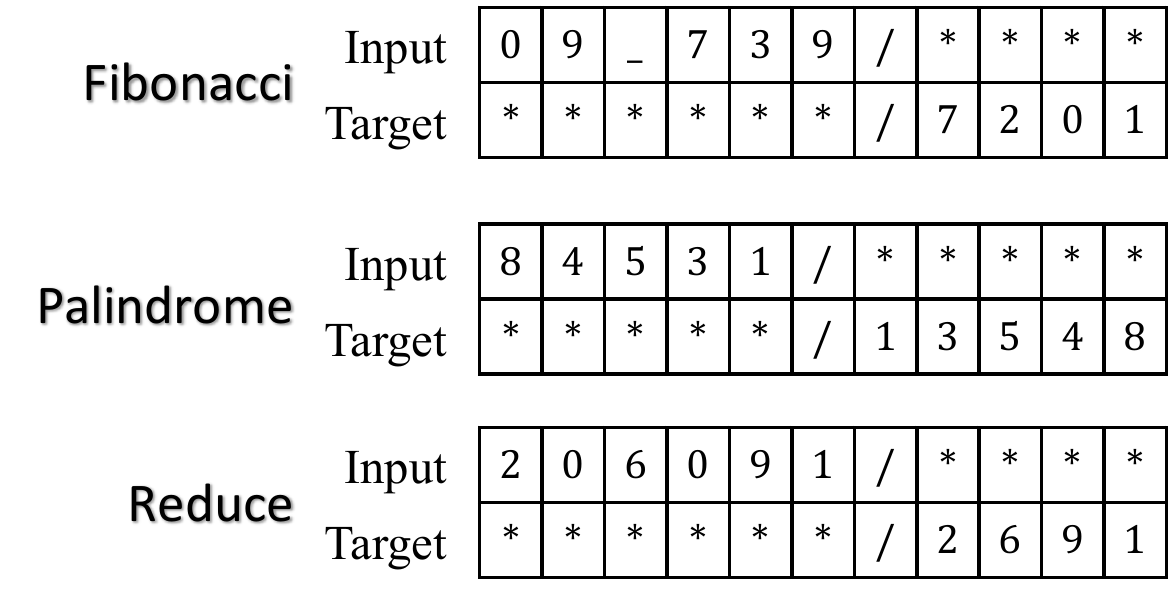}
    \caption{Input and output sequence examples of the algorithmic tasks. The Fibonacci sequence is given in the little-endian order.}
  \label{fig:tasks}  
  \vskip -0.1in
\end{figure}
\subsection{Algorithmic Tasks}
We test the computational powers of LSAM and NAM-TM using three types of algorithmic tasks in inductive zero-shot generalization setups.
First, number sequence prediction~\citep{numbersequenceprediction} task (NSP) is a suite of synthetic problems to predict the following digits of the numerical sequences.
It can test the inductive zero-shot generalization capability by testing/validating the models with the longer decimal numbers that are never observed during the training stage.
In this setup, many models often suffer from drastic fall of test/validation accuracy, due to lack of inductive bias~\citep{seq2grid}.
We use two representative sequences from NSP: Fibonacci (Fib) and Palindrome (Palin).
The two tasks require the generalization of digital addition and sequence reversal rules, respectively.

Next is a sequence reduction (Reduce) task with a simple rule: given the sequence of digits, the target is the reduced sequence by skipping zeros.
This is a task that can be easily solved by a proper Turing machine.
We also use a similar zero-shot generalization setup by testing/validating using the longer target sequences that are longer than any target sequences in the training dataset.

For NSP and Reduce tasks, we use training datasets with $d=1\dots 10$ decimal digit sequences in little-endian order.
Then we validate/test the models with two validation sets (ID, OD-easy) and a test set (OD-hard) for each task.
An in-distribution (ID) validation set consists of $d=5\dots 10$-digit sequences, and an out-of-distribution validation set (OD-easy) consists of $d=11\dots 13$-digit sequences.
The harder out-of-distribution test set (OD-hard) consists of $d=14\dots 16$-digit sequences, challenging the models with longer contexts.
A training dataset has 25600 samples, and each validation/test set has 2048 samples.

We format the problems as masked sequence completion problems as shown in Figure~\ref{fig:tasks}.
An input sequence consists of input tokens followed by masks, and an output sequence consists of target tokens following the masks.
Since the input and output sequences are 1:1 matched, we can compare to baseline models without sequence-to-sequence structures like DNC.
We choose four baseline models to compare: a bidirectional Transformer encoder (TF)~\citep{bert}, an 2-layer LSTM model with attention~\citep{bahdanau}, an Universal Transformer~\citep{universaltransformer}, and a Differentiable Neural Computer (DNC).
The TF model follows the architecture of BERT$_{medium}$ and the hyperparameters of the other models are adjusted to have similar parameter counts.
Our LSAM and NAM-TM networks have two LSAM layers and two NAM-TM layers respectively, whose hyperparameters are also adjusted to have similar parameter counts.
As an ablation study, we also evaluate NAM-TM design without the JUMP transition (No Jmp).
All experiments are run with PyTorch 1.10 on the system with Ubuntu 20.04 and RTX 3080.
Each experiment takes less than five hours to run 200 training epochs.
The implementation details and hyperparameters can be found at the source code in the supplementary materials.

\begin{table*}[ht]
\caption{Sequence accuracy (\%) comparison on the algorithmic tasks.}
\vskip 0.1in
\label{table:compgen}
\renewcommand{\arraystretch}{1.2}
\centering
\begin{tabular}{rrccccccc}
\hline
\multicolumn{2}{r}{Model}              & TF     & LSTM   & UT     & DNC    & LSAM            & NAM-TM & No Jmp\\ \hline
\multicolumn{2}{r}{Parameters}         & 28.7M  & 31.5M  & 26.0M  & 40.2M  & 25.8M           & 23.3M & 22.2M \\ \hline
\multirow{3}{*}{Palin}  & ID      & 100\%  & 100\%  & 100\%  & 100\%  & 100\%           & 100\% & 100\% \\
                             & OD-easy & 65.6\% & 100\%  & 99.2\% & 100\%  & 100\%           & 100\% & 100\% \\
                             & OD-hard & 19.0\% & 100\%  & 5.2\%  & 100\%  & 100\%           & 100\% &  100\% \\ \hline
\multirow{3}{*}{Fib}   & ID      & 98.9\% & 46.2\% & 99.1\% & 46.9\% & 100\%           & 97.4\% & 40.4\% \\
                             & OD-easy & 0.4\%  & 8.3\%  & 21.3\% & 1.8\%  & 39.9\%          & \textbf{89.7\%} & 19.8\%\\
                             & OD-hard & 0.0\%  & 0.0\%  & 0.1\%  & 0.0\%  & 2.9\%           & \textbf{71.5\%} &1.1\%\\  \hline
\multirow{3}{*}{Reduce}      & ID      & 99.6\% & 99.5\% & 98.7\% & 100\%  & 99.9\%          & 100\% &100\%\\
                             & OD-easy & 0.2\%  & 91.5\% & 15.7\% & 95.7\% & 96.8\%          & \textbf{100\%}  &\textbf{100\%}\\
                             & OD-hard & 0.0\%  & 63.6\% & 0.0\%  & 60.1\% & 77.6\%          & \textbf{100\%}  &\textbf{100\%}\\ \hline
\end{tabular}
\vskip -0.1in
\end{table*}

\subsection{Few-shot Learning}

We conduct $N$-way $K$-shot image classification task on MiniImageNet dataset.
Base on the code of Gidaris et al.~\footnote{https://github.com/gidariss/FewShotWithoutForgetting}, we implement our NAM few-shot learning algorithm and compare it to their cosine classifier.
We use two-layer feed-forward networks for determining the read and write probabilities of each sample.
Feature-averaging weights and attention-based mechanisms~\cite{fewshotwithoutforgetting} of the original work are applied to both cosine classifier and NAM few-shot learning, as they are independent to those methods.
Except for the few-shot weight generation module, we follow the same experimental configuration of the original work.
We use their ResNet-like feature extraction model because it shows the best classification results when the both novel and base categories are present.
We choose the best classifier by the validation accuracy for both categories (AccuracyBoth), and report the evaluation results with the test dataset.
As of the algorithmic tasks, experiments are run with PyTorch 1.10 on the system with Ubuntu 20.04 and RTX 3080, taking less than five hours per training one model.
The implementation details are available at the source code in the supplementary materials.

\subsection{Long-range task evaluation}

Although NAM is theoretically capable of replacing scaled dot-product attention, there exists a danger of information loss due to the limited capacity of the memory $M$.
Meanwhile, such an information loss is not an issue for Transformers because they utilize hidden activations from the entire sequence.
Hence we need empirical evidences that NAM is an effective alternative of attention.
As a proof-of-concept, we prove the efficacy and efficiency of NAM by evaluating NAM-Transformer at long-range sequence tasks.

We test our NAM-Transformer on long-range arena tasks~\citep{longrangearena} that have been used to compare efficient Transformer architectures.
Since the previous work has already proven that Linear Transformer is as capable as others, we conduct partial comparison of ours to the original Transformer and Linear Transformer.

We use the code base of the original benchmark~\footnote{github.com/google-research/long-range-arena} and implemented our NAM encoder by modifying the Linear Transformer implementation.
The three models share the same hyperparameters and the only differences come from the attention algorithms.
The evaluations are run on the Ubuntu 20.04 system with RTX 3080.
The results differ from the original work because we used smaller batch sizes due to the limited VRAM capacity.
All other experimental setups, including hyperparameters, are identical to the original benchmark~\citep{longrangearena}.
Details are available in the source code included as a supplementary material.

%% file: 5_result.tex
\begin{table*}[ht]
\caption{Top-1 test accuracies of cosine classifier (Cosine) and NAM few-shot learning (NAM). We report cases with only novel categories (Novel), base categories (Base), and both novel and base categories (Both). We also show 95\% confidence intervals.}
\vskip 0.1in
\label{table:fewshot}
\renewcommand{\arraystretch}{1.2}
\centering
\begin{tabular}{rcccccc}
\hline
           & \multicolumn{3}{c}{5-way 1-shot}                 & \multicolumn{3}{c}{5-way 5-shot}                           \\
Method     & Both               & Novel              & Base    & Both                        & Novel              & Base    \\ \hline
Cosine     & 52.68\% $\pm$ 0.43 & 56.00\% $\pm$ 0.82 & 79.86\% & 57.86\% $\pm$ 0.39          & 70.25\% $\pm$ 0.64 & 79.37\% \\
NAM (ours) & 52.88\% $\pm$ 0.43 & 56.11\% $\pm$ 0.81 & 79.55\% & \textbf{60.15\% $\pm$ 0.38} & 70.65\% $\pm$ 0.67 & 78.87\% \\ \hline
\end{tabular}
\vskip -0.1in
\end{table*}
\subsection{Algorithmic Task Evaluation}

Table~\ref{table:compgen} shows the sequence accuracy comparison of the models in algorithmic tasks.
We report the sequence accuracies of the models, where a wrong prediction of one token is counted as a failed prediction of the entire sequence.
To avoid cherry-picking, the epochs of the best OD-easy validation accuracies are presented for the evaluation results.

Although LSAM's architecture is not specifically designed for algorithmic tasks, LSAM performs better than the other baselines.
Surprisingly, it performs better than DNC, which is a MANN model designed for such algorithmic problems.
This is not a matter of model capacity because LSAM has a slightly smaller parameter count.
The results imply that the computational power of NAM architecture is superior to that of DNC's memory architecture.

As expected, NAM-TM performs significantly better in algorithmic tasks (Palin, Fib, and Reduce) possibly due to its specialized architecture.
Especially, it finds easier to solve OD-hard problems, whereas the other models experience steep performance decline.
A potential explanation is that the action-based positional transitions provide robustness in long-context cases.

NAM-TM remains effective at palindrome and reduction tasks even without the JUMP transition (No Jmp).
Hence, LEFT and RIGHT transitions seem to be enough for emulating simple Turing machines, but more complex transition rules can augment the computational power of NAM-TM.
This suggests extending NAM-TM to variety of tasks by augmenting specialized transition actions.

\subsection{Few-shot Learning Evaluation}

We compare our few-shot learning method to the feature-averaging plus attention-based weight generator at $N$-way $K$-shot learning task with MiniImageNet dataset~\cite{fewshotwithoutforgetting}.
Following the original work, we compare their top-1 accuracy in 5-way 1-shot and 5-way 5-shot setups.
Table~\ref{table:fewshot} shows the comparison between the baseline cosine few-shot classifier and our NAM few-shot classifier.
Both methods share the same ResNet-like feature extractor and attention-based weight generation mechanism.
For most of the cases, their performance differences are well within the confidence intervals.
However, the only experiments that are different beyond the confidence intervals are accuracies of both categories (Both) in the 5-shot setup.
Overall, the accuracies of both categories are less than accuracies of novel or base categories, suggesting that the few-shot classifiers suffer from false positives when both of novel and base categories are present.
NAM few-shot shows higher both-category accuracy in the 5-shot setup and slightly lower base-category accuracy compared to the baseline.
This indicates that the erasure capability of NAM mitigates the false-positive problem, with minimal forgetting for the base classes.
However, it seems like this effect requires numbers of samples as NAM and the baseline are almost identical at the one-shot setup.

\subsection{Long-range Task Efficiency}

\begin{table}[t]
\caption{Accuracy (\%) and relative training speedup comparison of the original Transformer (TF), Linear Transformer (Linear) and our NAM Transforemr (NAM) in listops (Listops), text classification (Text), pixel-level image classification (Image) tasks.}
\vskip 0.1in
\label{table:longrange}
\renewcommand{\arraystretch}{1.2}
\centering
\begin{tabular}{rcccccc}
\hline
                                                             & \multicolumn{2}{c}{ListOps} & \multicolumn{2}{c}{Text} & \multicolumn{2}{c}{Image} \\
Model                                                        & Acc          & Spd          & Acc         & Spd        & Acc         & Spd         \\ \hline
TF                                                  & 29.8         & 1            & 57.3        & 1          & 40.16       & 1           \\
Linear & 36.5         & 2.59         & 64.3        & 2.35       & 38.76       & 10.02       \\
NAM    & 36.1         & 2.66         & 63.5        & 2.44       & 37.18       & 10.07       \\ \hline
\end{tabular}
\vskip -0.1in
\end{table}

Table~\ref{table:longrange} shows that NAM can be an efficient alternative to Transformers.
Despite the information loss and missing normalizing factor $Z$, the task accuracy of NAM is very similar to the others, even surpassing the original Transformer in some cases.
We see bigger speedups at the image classification task because the models have much smaller per-head dimension than others (16 vs 64).
Recall that the computational complexities of both linear transformer and NAM transformer are quadratic to the per-head dimensions.
Hence, the image classification result shows that smaller per-head dimension can bring bigger speedups for NAM, but may result in lesser capacity.
Overall, the normalized outer-product attention is as powerful as the scaled dot-product attention while enjoying greater computational efficiency.
In other words, NAM can be an efficient and effective alternative for attention mechanism.

%% file: 6_discussion.tex


\paragraph{Hierarchical data modeling}
Tensor product, the key operation behind NAM, is not limited to two-dimensional outer product.
Any dimensional tensors $T_i$ can construct a memory tensor $M$ by performing sum of tensor products $ M = \sum_i{T_i \otimes k_i} $, with the unit key vectors $k_i$.
For example, a document tensor $D$ can be constructed by a sum of tensor products of sentence-level keys and sentence matrices $S_i$, each of which is also a sum of outer products of word embeddings and word-level keys.
\emph{Nested attention} to such a hierarchical memory can be conducted by two inner products with a sentence-level unit query vector $q_s$ and a word-level unit query vector $q_w$.

\paragraph{Efficient edge inference}
While Transformers are very successful in many ML tasks, deploying such models for edge inference has been a challenging task since Transformer's computation and memory cost per each time step varies with the sequence length.
However, NAM's cost does not depend on the sequence length at all.
Also, the outer product operation is more compute intensive than the dot products, making NAM more friendly to high-throughput accelerators.
Hence, NAM can be an efficient Transformer alternative for edge inference.


%% file: 7_conclusion.tex
We proposed a redesign of attention mechanism to construct a differentiable memory for neural networks, namely neural attention memory (NAM).
Following the same query-key-value structure of scaled dot-product attention, NAM first writes the memory matrix by adding outer products of key-value pairs.
Then we can read it by multiplying the matrix with a unit query vector. 
First, NAM can be a powerful basis for constructing MANN models.
We designed two NAM-based MANNs: LSAM for generic sequential tasks and NAM-TM for algorithmic tasks.
In algorithmic tasks that require inductive zero-shot generalization, both outperformed other baselines such as Universal Transformer and DNC, indicating that NAM is a more powerful mechanism for implementing memory in DNNs.
Next, memorization with NAM can be leveraged as a few-shot learning mechanism.
We showed that NAM based few-shot learning is an effective method for N-way K-shot learning, outperforming SOTA for mitigating false-positives.
Also, it can construct an efficient linear attention mechanism as its computational complexity does not rely on the sequence length.
The long-range arena evaluations showed that NAM based attention, namely normalized outer-product attention, is an efficient and effective alternative for scaled dot-product attention.